\documentclass{article}

\usepackage[final]{neurips_2024}

\usepackage[utf8]{inputenc}
\usepackage[T1]{fontenc}
\usepackage{url}
\usepackage{booktabs}
\usepackage{amsfonts}
\usepackage{amsmath}
\usepackage{amssymb}
\usepackage{amsthm}
\usepackage{nicefrac}
\usepackage{microtype}
\usepackage{graphicx}
\usepackage{subcaption}
\usepackage{xcolor}
\usepackage{float}
\usepackage{pmboxdraw}
\usepackage{newunicodechar}
\usepackage{hyperref}

\DeclareMathOperator{\ReLU}{ReLU}
\DeclareMathOperator{\tr}{tr}
\DeclareMathOperator{\diag}{diag}

\newtheorem{proposition}{Proposition}

\newunicodechar{└}{\textSFii}
\newunicodechar{├}{\textSFviii}
\newunicodechar{─}{\textSFx}
\newunicodechar{⋮}{\ensuremath{\vdots}}

\title{I Dropped a Neural Net}

\author{%
  Hyunwoo Park \\
  Carnegie Mellon University\\
  \texttt{hp2@andrew.cmu.edu}
}
\begin{document}

\maketitle

\begin{abstract}
A recent Dwarkesh Patel podcast with John Collison and Elon Musk featured an interesting puzzle from Jane Street: \emph{they trained a neural net, shuffled all 96 layers, and asked  to put them back in order}.

Given unlabelled layers of a Residual Network and its training dataset, we recover the
exact ordering of the layers. The problem decomposes into pairing each
block's input and output projections ($48!$ possibilities) and ordering
the reassembled blocks ($48!$ possibilities), for a combined search space
of $(48!)^2 \approx 10^{122}$, which is more than the atoms in the observable universe.
We show that stability conditions during training like dynamic isometry leave the product $W_{\text{out}} W_{\text{in}}$
for correctly paired layers with a negative diagonal structure, allowing us to use
\emph{diagonal dominance ratio} as a signal for pairing. For ordering, we seed-initialize with a rough proxy such as delta-norm or $\|W_{\text{out}}\|_F$  then hill-climb to zero mean squared error.
\end{abstract}
\section{Introduction}
\label{sec:intro}

Consider the following combinatorial puzzle: a 48-block ResNet has
been ``dropped,'' scattering its 97 constituent linear layers (96~block layers
plus one output projection) into an unlabelled collection. Given only these
and 10{,}000 training datapoints, can we recover
the original network?

Deep Residual Networks~\cite{he2016deep} compute a sequence of near-identity
transformations, where each block adds a small perturbation to its input via a
residual branch. Veit et al.~\cite{veit2016residual} showed that
this architecture behaves like an ensemble of relatively shallow networks,
implying that individual blocks can be reordered with controlled
impact on the output. A complementary perspective comes from dynamic
isometry~\cite{saxe2014exact,pennington2017resurrecting,xiao2018dynamical}:
well-trained residual blocks have stable gradients, which constrains the
structure of their weight products. We exploit these properties to \emph{reconstruct} a dropped ResNet.

The problem decomposes into \emph{pairing} (matching each input projection to its output projection) and \emph{ordering} (sequencing the 48 reassembled blocks), yielding a combined search space of $(48!)^2 \approx 10^{122}$, more than the atoms in the observable universe, which makes brute force impossible.

\begin{enumerate}
    \item \textbf{Pairing} (Section~\ref{sec:pairing}): Match each $W_{\text{in}}$
    to its $W_{\text{out}}$ via a diagonal dominance ratio on $W_{\text{out}} W_{\text{in}}$.
    \item \textbf{Ordering} (Section~\ref{sec:ordering}): Recover the block
    sequence via a seed initialization, refined
    by Bradley--Terry ranking, and finalized by MSE bubble repair.
\end{enumerate}

\newpage 
\section{The Puzzle}
\label{sec:puzzle}

\begin{quote}
\emph{Oh no! I dropped an extremely valuable trading model and it fell apart into linear layers! I need to rebuild it before anyone notices, but I can't remember how these pieces go together, or how it was trained.
All I have left are the pieces of the model and some historical data. Can you help me figure out how to put it back together?
Luckily I still have the source code of the layers that the neural network is made of. They look like this:}
\end{quote}

\begin{verbatim}
class Block(nn.Module):
    def __init__(self, in_dim: int, hidden_dim: int):
        super().__init__()
        self.inp = nn.Linear(in_dim, hidden_dim)
        self.activation = nn.ReLU()
        self.out = nn.Linear(hidden_dim, in_dim)

    def forward(self, x):
        residual = x
        x = self.inp(x)
        x = self.activation(x)
        x = self.out(x)
        return residual + x

class LastLayer(nn.Module):
    def __init__(self, in_dim: int, out_dim: int):
        super().__init__()
        self.layer = nn.Linear(in_dim, out_dim)

    def forward(self, x):
        return self.layer(x)
\end{verbatim}
The solution is a permutation of indices $\{0, \ldots, 96\}$.
We are also given :
\begin{verbatim}
historical_data_and_pieces
├── historical_data.csv
└── pieces
    ├── piece_0.pth
    ├── piece_1.pth
    ⋮
    └── piece_96.pth
\end{verbatim}

\noindent The CSV file contains 10{,}000 rows with 48 input features
(\texttt{measurement\_0} through \texttt{measurement\_47}),
the original model's predictions (\texttt{pred}),
and the ground-truth target values (\texttt{true}).

\newpage


\subsection{Architecture}

The network consists of 48 residual blocks followed by a final linear
layer (Figure~\ref{fig:architecture}). Each block computes:
\begin{equation}
    \text{Block}_k(x) = x + r_k(x), \qquad
    r_k(x) := W_{\text{out}}^{(k)} \ReLU\!\big(W_{\text{in}}^{(k)} x + b_{\text{in}}^{(k)}\big) + b_{\text{out}}^{(k)}
    \label{eq:block}
\end{equation}
where $W_{\text{in}}^{(k)} \in \mathbb{R}^{96 \times 48}$,
$W_{\text{out}}^{(k)} \in \mathbb{R}^{48 \times 96}$,
and the \texttt{in\_dim=48}  with \texttt{hidden\_dim=96}.
The full forward pass is:
\begin{equation}
    f_\sigma(x) = W_{\text{last}}\, \text{Block}_{\sigma(48)} \circ \cdots \circ \text{Block}_{\sigma(1)}(x) + b_{\text{last}}
\end{equation}
where $\sigma \in S_{48}$ is the block permutation and
$W_{\text{last}} \in \mathbb{R}^{1 \times 48}$ is the scalar output
projection (\texttt{piece\_85}, identified by its unique shape).

\begin{figure}[H]
    \centering
    \includegraphics[height=0.70\textheight]{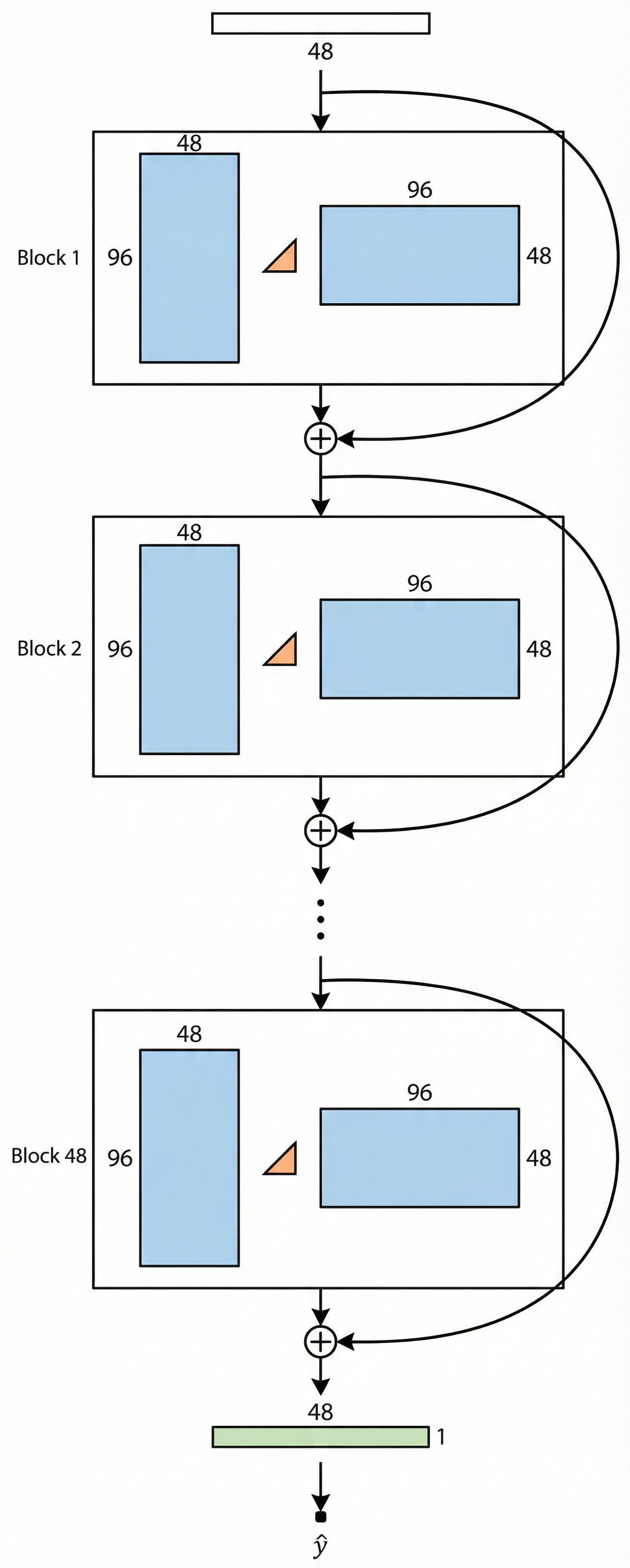}
    \caption{\textbf{Architecture.}
    The 48-block ResNet. Each block applies
    input layer,
    ReLU, output layer,
    then a residual connection. The final
    layer produces the scalar output $\hat{y}$.}
    \label{fig:architecture}
\end{figure}

\subsection{Available Data}
\label{sec:data}

We are given:
\begin{itemize}
    \item 97 weight matrices and bias vectors.
    The shapes are: 48 of $(96, 48)$ (input projections),
    48 of $(48, 96)$ (output projections), and 1 of $(1, 48)$ (last layer).
    \item $N = 10{,}000$ data points $(x_i, \hat{y}_i, y_i)$ where
    $x_i \in \mathbb{R}^{48}$ are input features,
    $\hat{y}_i \in \mathbb{R}$ are the original model's predictions, and
    $y_i \in \mathbb{R}$ are the ground-truth targets.
\end{itemize}

\begin{figure}[ht]
    \centering
    \begin{subfigure}[t]{0.48\textwidth}
        \centering
        \includegraphics[width=\textwidth]{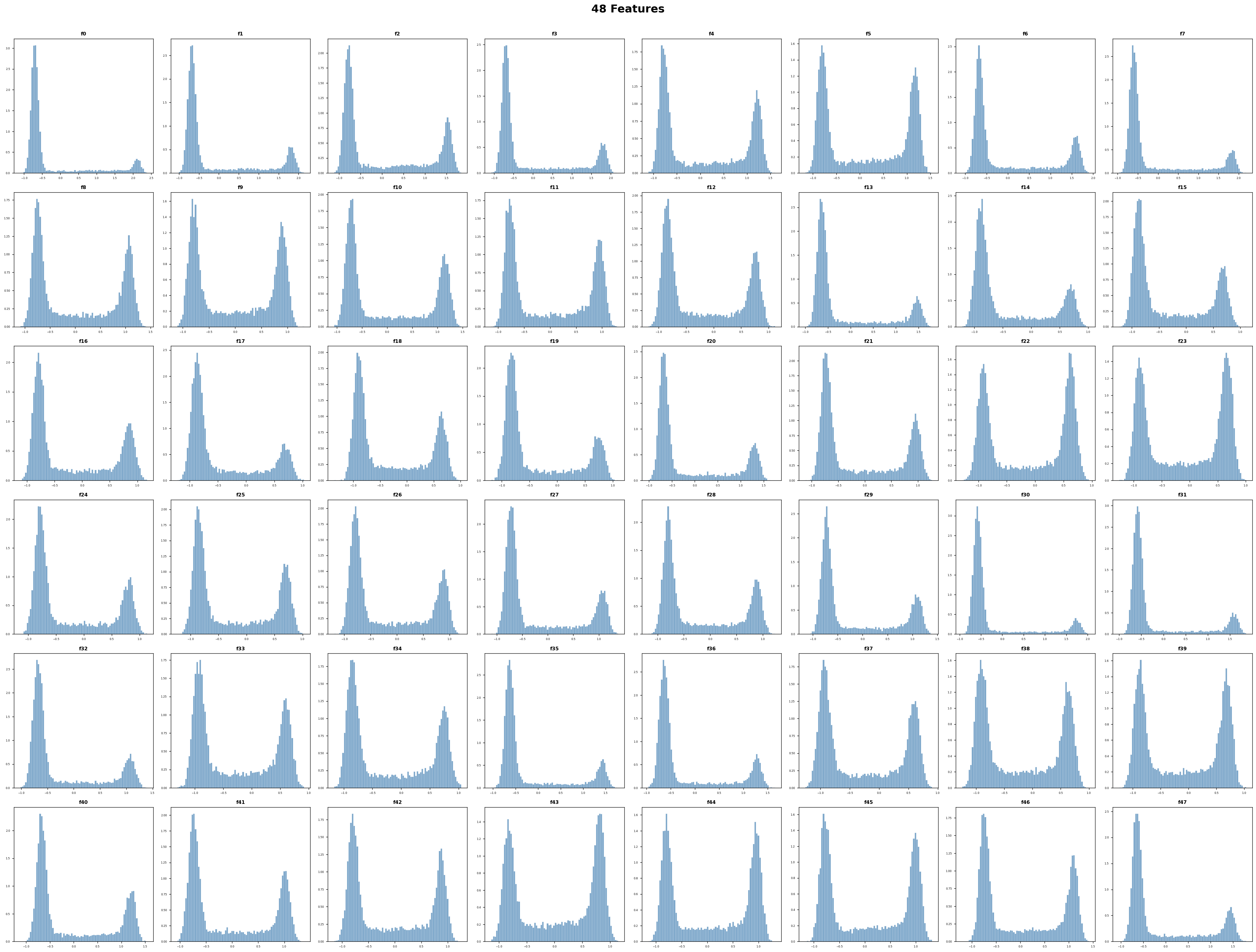}
        \caption{Sorted feature values.}
        \label{fig:sorted_48_features}
    \end{subfigure}
    \begin{subfigure}[t]{0.48\textwidth}
        \centering
        \includegraphics[width=\textwidth]{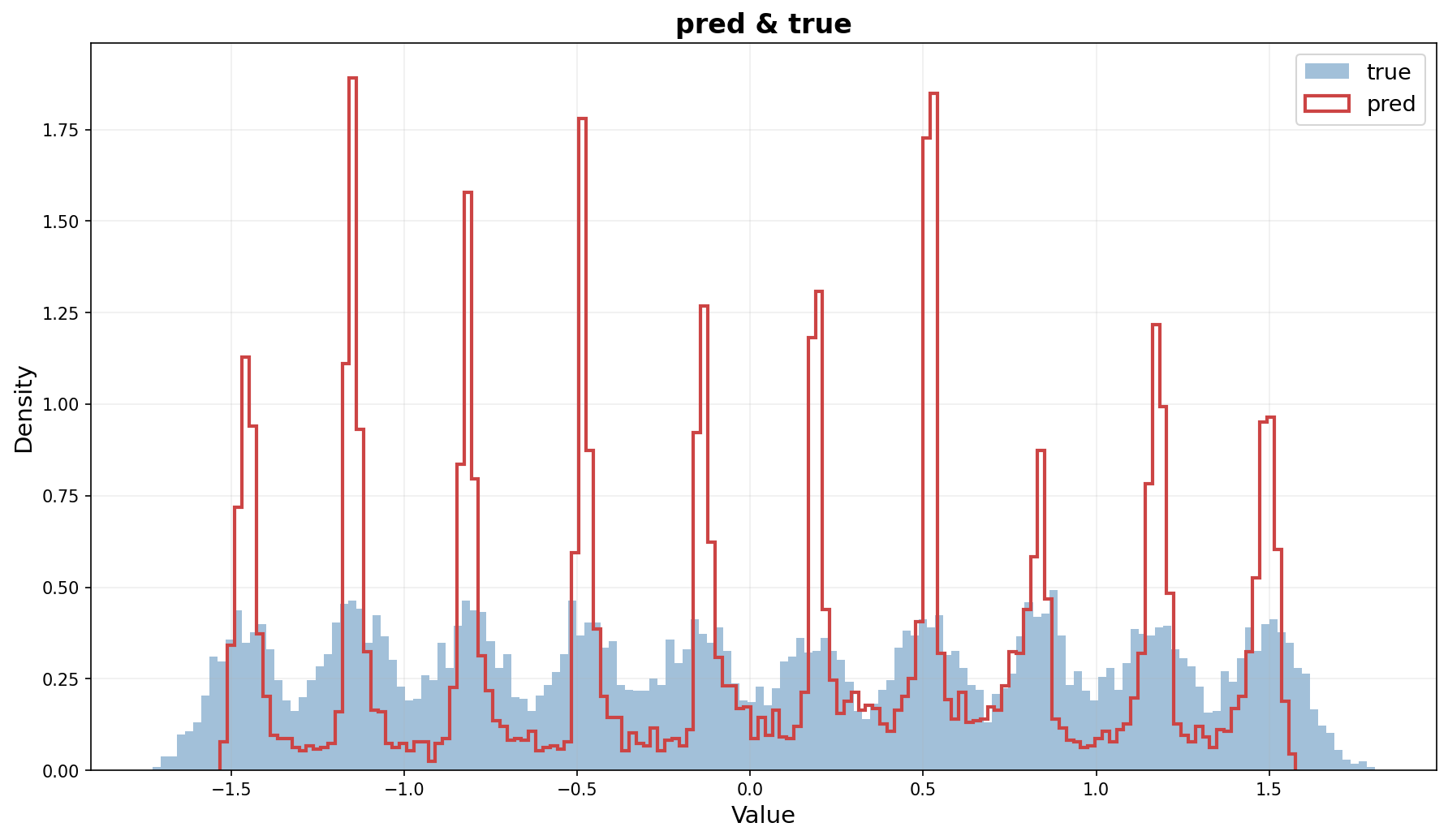}
        \caption{Sorted \texttt{pred} and \texttt{true}.}
        \label{fig:sorted_pred_true}
    \end{subfigure}\hfill
    \caption{\textbf{Dataset overview.} Features and outputs are normalized. }
    \label{fig:data_overview}
\end{figure}

\subsection{Objective}

Given the dataset $\{(x_i,\hat{y}_i)\}_{i=1}^N$ ,
we define the mean-squared error against \texttt{pred}
\begin{equation}
   \mathcal{L}(\sigma)
   \;:=\; \frac{1}{N}\sum_{i=1}^{N}\left(f_\sigma(x_i) - \hat{y}_i\right)^2.
   \label{eq:empirical_mse}
\end{equation}
The correct permutation should have $\mathcal{L}(\sigma)=0$ up to floating-point error.

\includegraphics[width=\textwidth]{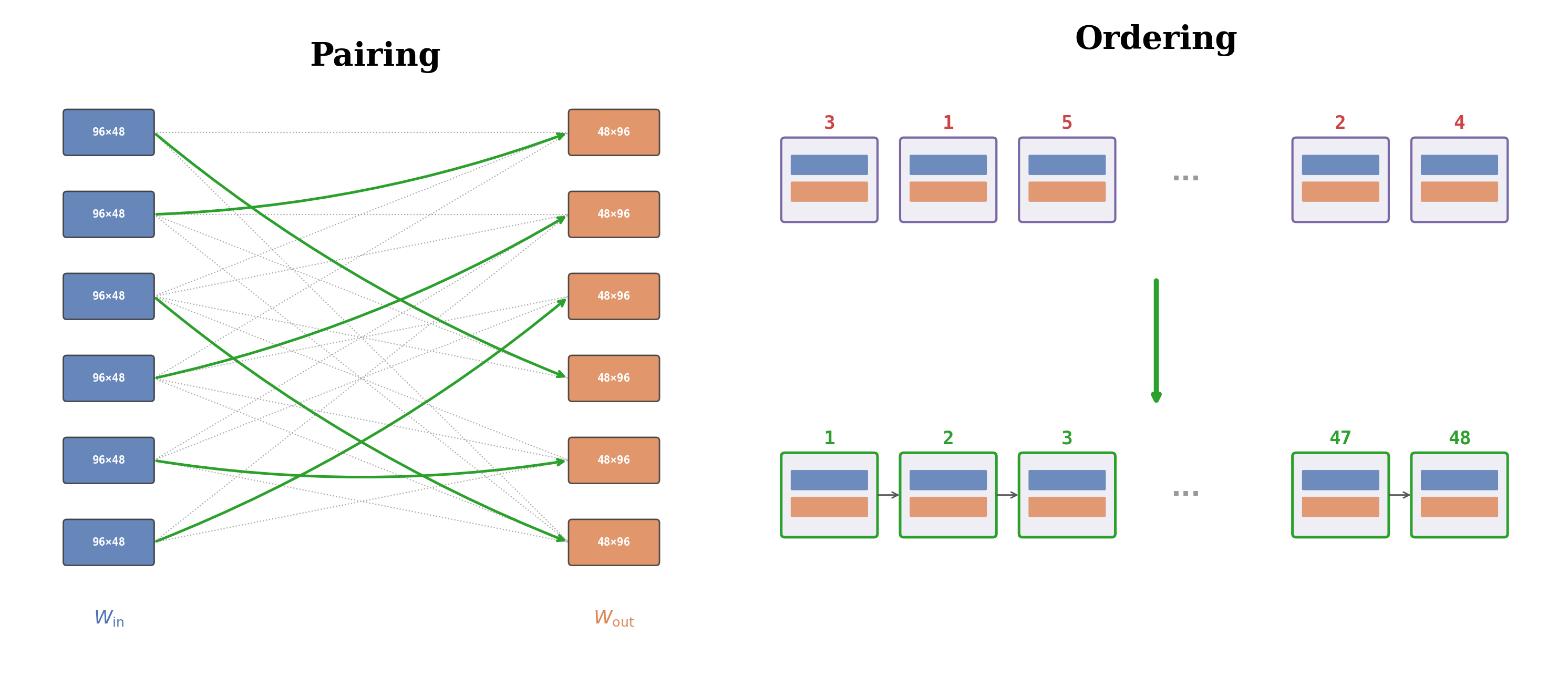}

\begin{enumerate}
    \item \textbf{Pairing:} Pair each $W_{\text{in}} \in \mathbb{R}^{96 \times 48}$
    to its $W_{\text{out}} \in \mathbb{R}^{48 \times 96}$. Bipartite matching with $48!$ possibilities.
    \item \textbf{Ordering:} Order the 48 paired blocks in the correct sequence.
    Permutation over another $48!$ possibilities.
\end{enumerate}

The combined search space of $(48!)^2 \approx 10^{122}$ is more than the atoms in the observable universe, making brute force search impossible.

\newpage 
\section{Pairing}
\label{sec:pairing}

We train a small model with the same architecture on the same data. Upon examining the product $W_{\text{out}}^{(k)} W_{\text{in}}^{(k)} \in \mathbb{R}^{48 \times 48}$ we observe a negative diagonal structure. (Proposition~\ref{prop:jacobian}).
 Which motivates the metric:

\subsection{Diagonal Dominance Ratio}

For each candidate pair of input layer $i$ and output layer $j$, we form
the product $W_{\text{out}}^{(j)} W_{\text{in}}^{(i)}$ and compute
the \textbf{diagonal dominance ratio}:
\begin{equation}
    d(i, j) = \frac{|\tr(W_{\text{out}}^{(j)} W_{\text{in}}^{(i)})|}{\|W_{\text{out}}^{(j)} W_{\text{in}}^{(i)}\|_F}
    \label{eq:diag_ratio}
\end{equation}

For correctly paired blocks, the negative diagonal structure
(Proposition~\ref{prop:jacobian}) contributes to the trace magnitude,
giving $d \in [1.76, 3.28]$.
For incorrect pairings, the product has no structure,
giving $d \in [0.00, 0.58]$.

We compute $d(i,j)$ for all $48 \times 48$ candidate pairs and solve the
optimal bipartite assignment via the Hungarian
algorithm~\cite{kuhn1955hungarian} in $O(n^3)$, recovering all 48 pairs
perfectly. In practice row-wise greedy
argmax is enough; the Hungarian algorithm simply provides the
optimality guarantee.

\begin{figure}[ht]
    \centering
    \begin{subfigure}[t]{0.48\textwidth}
        \includegraphics[width=\textwidth]{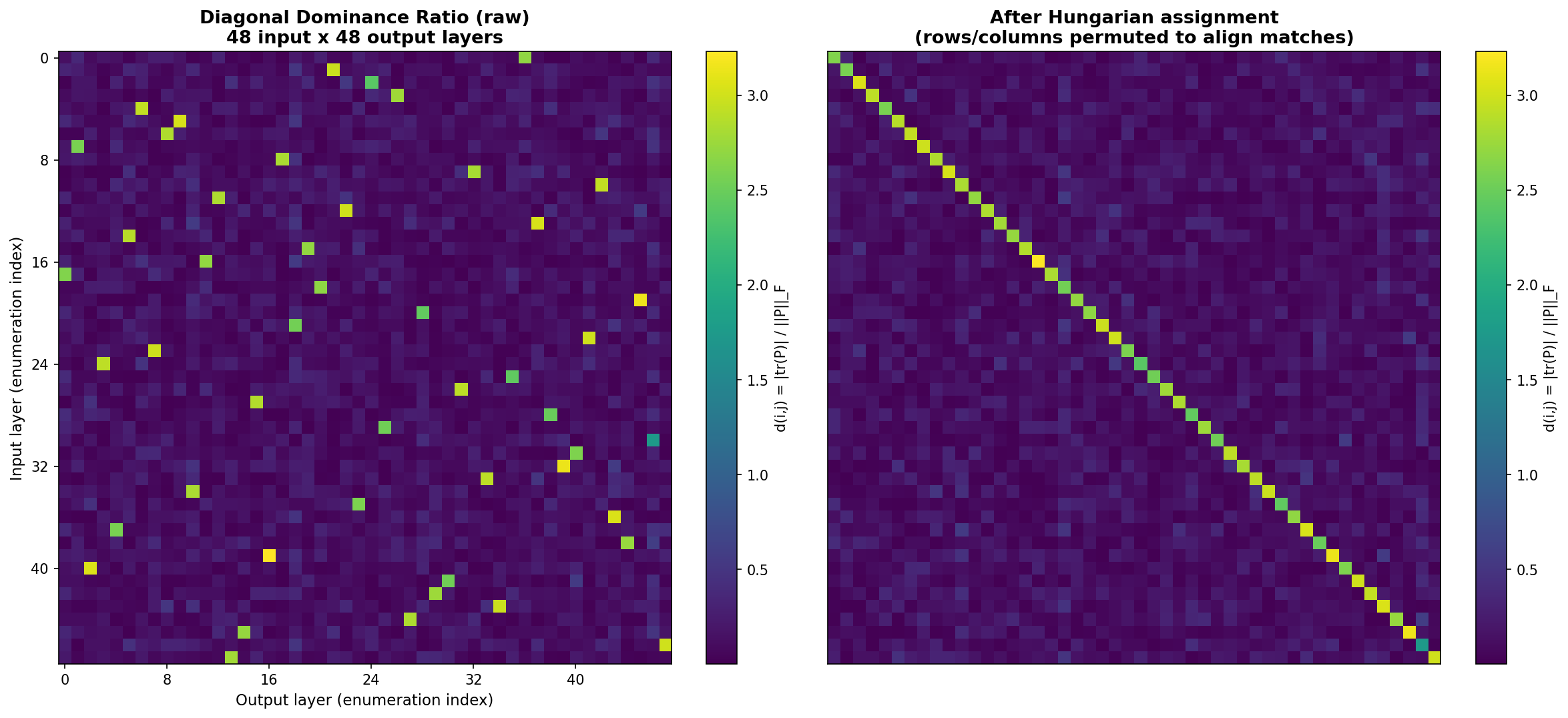}
        \caption{Left: raw $d(i,j)$; one signal per row.
        Right: after matching.}
        \label{fig:pairing_process}
    \end{subfigure}\hfill
    \begin{subfigure}[t]{0.48\textwidth}
        \includegraphics[width=\textwidth]{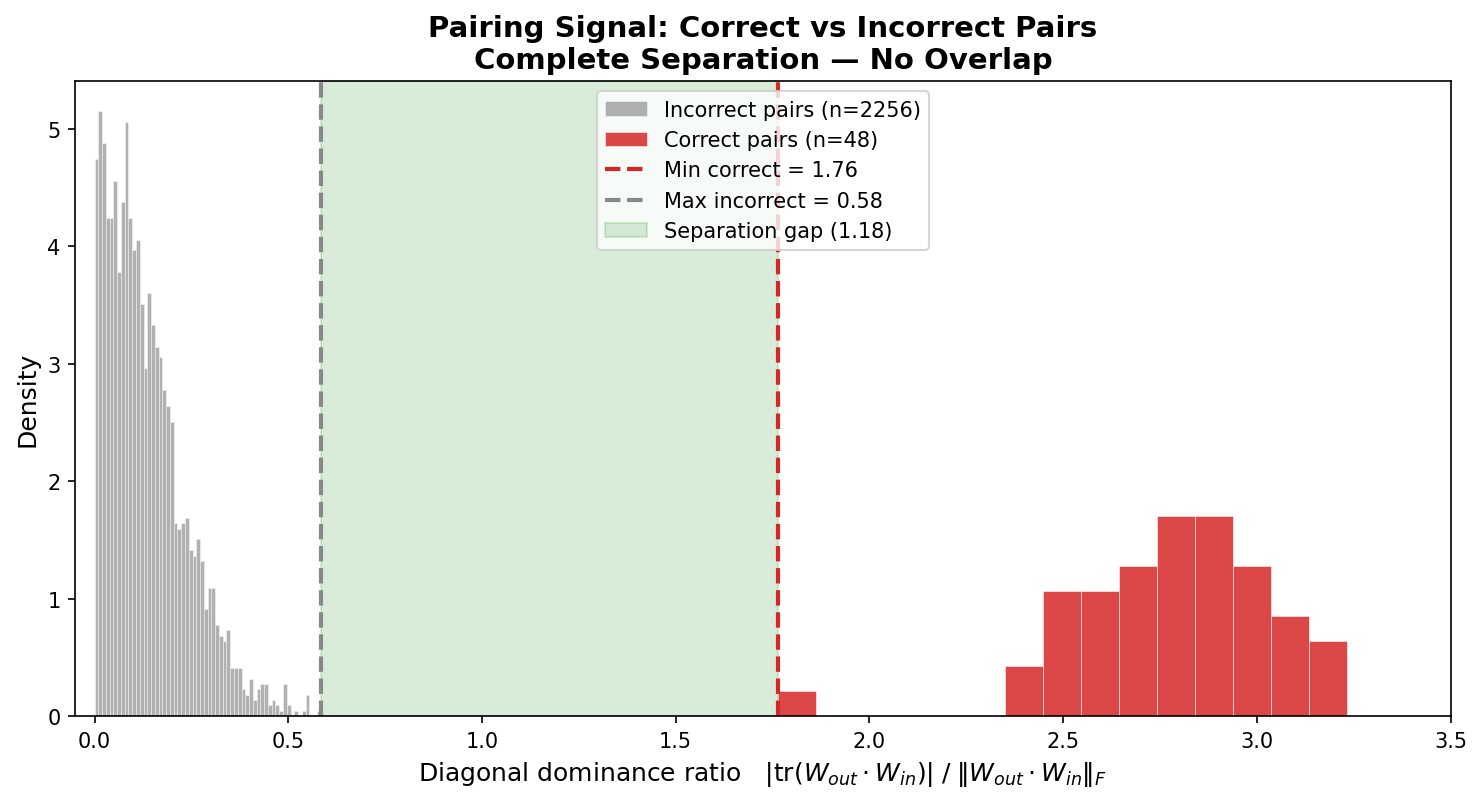}
        \caption{Complete separation: 48 correct pairs (red, $d \geq 1.76$)
        vs.\ 2{,}256 incorrect (gray, $d \leq 0.58$).}
        \label{fig:pairing_histogram}
    \end{subfigure}
    \caption{\textbf{Pairing via diagonal dominance.}
    The diagonal dominance ratio achieves complete separation between
    correct pairs ($d \in [1.76, 3.28]$) and
    incorrect pairs ($d \in [0.00, 0.58]$), with a gap of $1.18$.}
    \label{fig:pairing_combined}
\end{figure}

\begin{figure}[H]
    \centering
    \includegraphics[width=\textwidth]{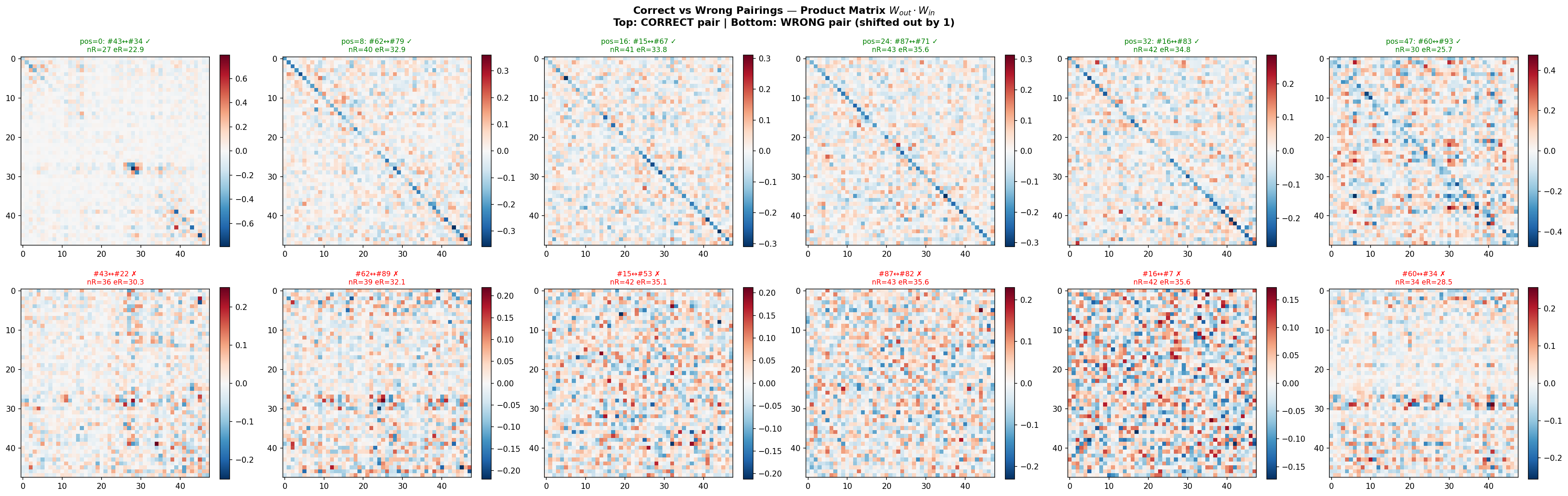}
    \caption{\textbf{Correct vs.\ incorrect pairings.}
    Top: Correctly paired matrices show negative diagonal structure.
    Bottom: Incorrect pairings with no structure.}
    \label{fig:correct_vs_wrong}
\end{figure}

\begin{figure}[H]
    \centering
    \includegraphics[width=\textwidth]{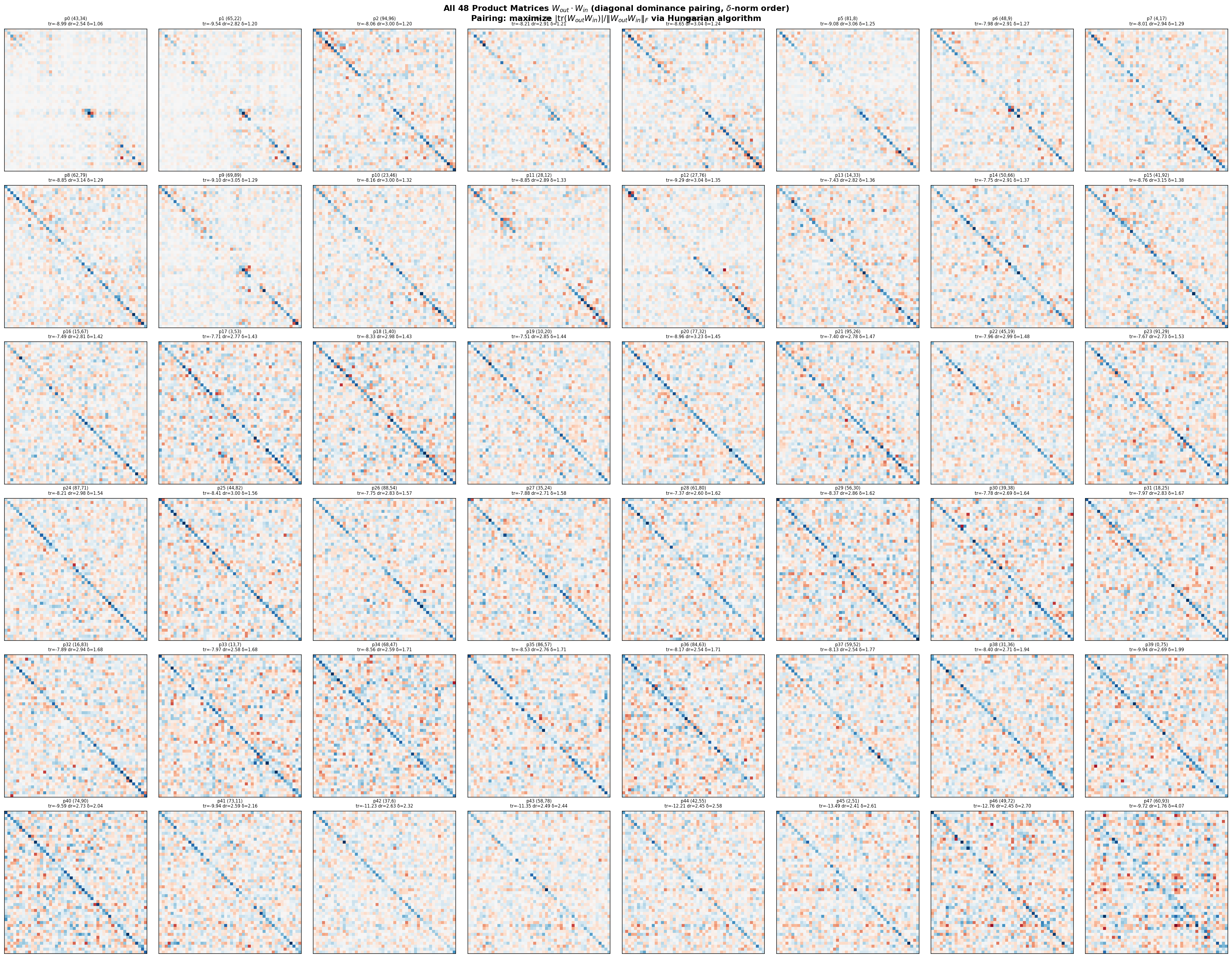}
    \caption{\textbf{All 48 product matrices from the original
     model,} ordered by delta-norm $\delta_k$
    (Section~\ref{sec:seed}). Every block exhibits a negative diagonal
    structure. Earlier blocks (top-left) and later blocks
    (bottom-right). Traces range from $-13.5$ to $-7.4$.
    Compare with Figure~\ref{fig:toy_matrices}.}
    \label{fig:all48}
\end{figure}
\section{Ordering}
\label{sec:ordering}

With all 48 correctly paired blocks, we now need to recover the correct order. We start by seed ordering with a rough proxy that achieves low MSE.

\subsection{Seed Ordering}
\label{sec:seed}

\paragraph{Delta-norm (data-dependent).}
We compute each block's \textbf{delta-norm}, the mean magnitude of its
residual contribution when applied to the raw input:
\begin{equation}
    \delta_k \;:=\; \mathbb{E}_{x \sim \mathcal{D}}\!\left[\|r_k(x)\|_2\right]
    = \mathbb{E}_{x \sim \mathcal{D}}\!\left[
    \Big\|W_{\text{out}}^{(k)}\,\ReLU\!\big(W_{\text{in}}^{(k)} x + b_{\text{in}}^{(k)}\big) + b_{\text{out}}^{(k)}\Big\|_2
    \right]
\end{equation}
Sorting by ascending $\delta_k$ places blocks with smaller perturbations first,
motivated by the empirical observation that residual perturbation magnitude
increases with depth in trained networks, inspired by the
rank-diminishing principle~\cite{feng2022rank}.
This achieves MSE~$= 0.036794$.

\paragraph{$\|W_{\text{out}}\|_F$ (data-free).}
We find that sorting blocks by the Frobenius norm of their output projection
$\|W_{\text{out}}^{(k)}\|_F$ is effective and data-free. Sorting by ascending $\|W_{\text{out}}\|_F$ achieves MSE~$= 0.075851$.

\subsection{Near-Commutativity of Residual Blocks}

We must examine the noise in ordering before we proceed with sorting. Because each block computes $x \mapsto x + r_k(x)$ where $r_k$ is small
relative to $x$, the composition of two blocks is approximately commutative:
\begin{equation}
    \text{Block}_A \circ \text{Block}_B(x) \approx \text{Block}_B \circ \text{Block}_A(x) + O(\|r_A\| \|r_B\|)
\end{equation}
Therefore pairwise ordering preferences are only approximately \textbf{transitive}:
if $A \prec B$ and $B \prec C$, then $A \prec C$ with high probability since the residual perturbations are small relative to the identity path. We can't rely on sorting algorithm alone, therefore we look to methods that are robust to noisy ordering.

\subsection{Bradley--Terry Ranking}
\label{sec:bt}

For all $\binom{48}{2} = 1{,}128$ pairs $(A, B)$, we swap their
positions in the seed ordering and compute the change in MSE:
\begin{equation}
    g_{AB} = \mathcal{L}(\sigma_{\text{swap}(A,B)}) - \mathcal{L}(\sigma_{\text{seed}})
\end{equation}
If $g_{AB} > 0$, then $A$ should precede $B$; if $g_{AB} < 0$, the reverse.
Each comparison is done on one forward pass with just $N_{\text{cmp}} = 2{,}000$ points for efficiency.

Rather than counting wins, we fit a margin-weighted Bradley--Terry
model~\cite{bradley1952rank} to extract strength parameters.
Each pairwise MSE gap $g_{AB}$ is converted to a soft win probability
via a sigmoid with temperature $T$:
\begin{equation}
    p_{A \prec B}  = \frac{1}{1 + \exp(-g_{AB}/T)}
\end{equation}

The latent strengths $\{s_k\}$ are estimated by maximizing the Bradley--Terry
likelihood via the MM algorithm of Hunter~\cite{hunter2004mm}:
\begin{equation}
    s_i^{(t+1)} = \frac{\sum_{j \neq i} w_{ij}}{\sum_{j \neq i} \frac{w_{ij} + w_{ji}}{s_i^{(t)} + s_j^{(t)}}}
\end{equation}
where $w_{ij} = p_{i \prec j}$, converging in $\sim$280 iterations ($T = 0.001$).

\paragraph{Transitivity violations.}
Out of $\binom{48}{3} = 17{,}296$ directed triples, only \textbf{66} (0.38\%) exhibit a
transitivity violation, forming 22 unique 3-cycles. These violations are
concentrated among a small number of blocks, spatially local
(spanning $< 15$ ground-truth positions), and weak
(median margin $1.8 \times 10^{-4}$ vs.\ overall median $0.041$).
Sorting by BT strength yields MSE~$= 0.00299$, a $12\times$ improvement
over the delta-norm seed.

\subsection{Bubble Repair}
\label{sec:repair}

We finally perform local correction by greedy adjacent-swap hill
climbing on MSE: sweep through all 47 adjacent pairs,
swap if it reduces MSE, undo otherwise, and repeat until a full sweep
produces zero swaps. 

\begin{table}[H]
    \centering
    \caption{\textbf{Bubble repair convergence from BT ordering.}
    Converges to MSE~$= 0$ in 5 rounds with 37 swaps.}
    \label{tab:repair}
    \begin{tabular}{cccc}
        \toprule
        \textbf{Round} & \textbf{Swaps} & \textbf{MSE} & \textbf{Cum.\ swaps} \\
        \midrule
        0 (BT init) & -- & 0.002989 & 0 \\
        1 & 21 & 0.001025 & 21 \\
        2 & 7 & 0.000553 & 28 \\
        3 & 5 & 0.000200 & 33 \\
        4 & 3 & 0.000052 & 36 \\
        5 & 1 & \textbf{0.000000} & 37 \\
        \bottomrule
    \end{tabular}
\end{table}

\newpage 
\subsection{Visual Progression}

\begin{figure}[H]
    \centering
    \begin{subfigure}[t]{\textwidth}
        \centering
        \includegraphics[width=0.72\textwidth]{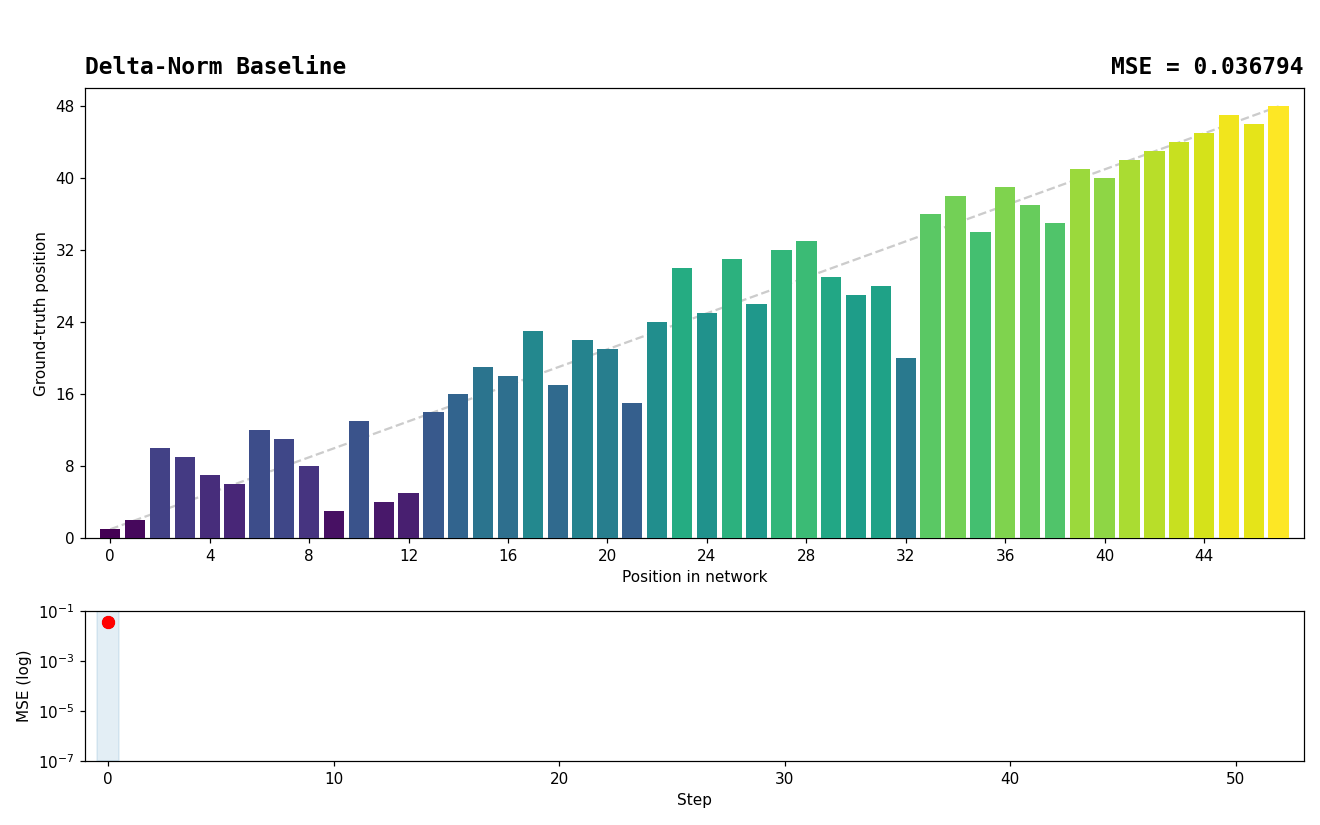}
        \caption{Delta-norm seed.}
    \end{subfigure}
    \vspace{-2pt}
    \begin{subfigure}[t]{\textwidth}
        \centering
        \includegraphics[width=0.72\textwidth]{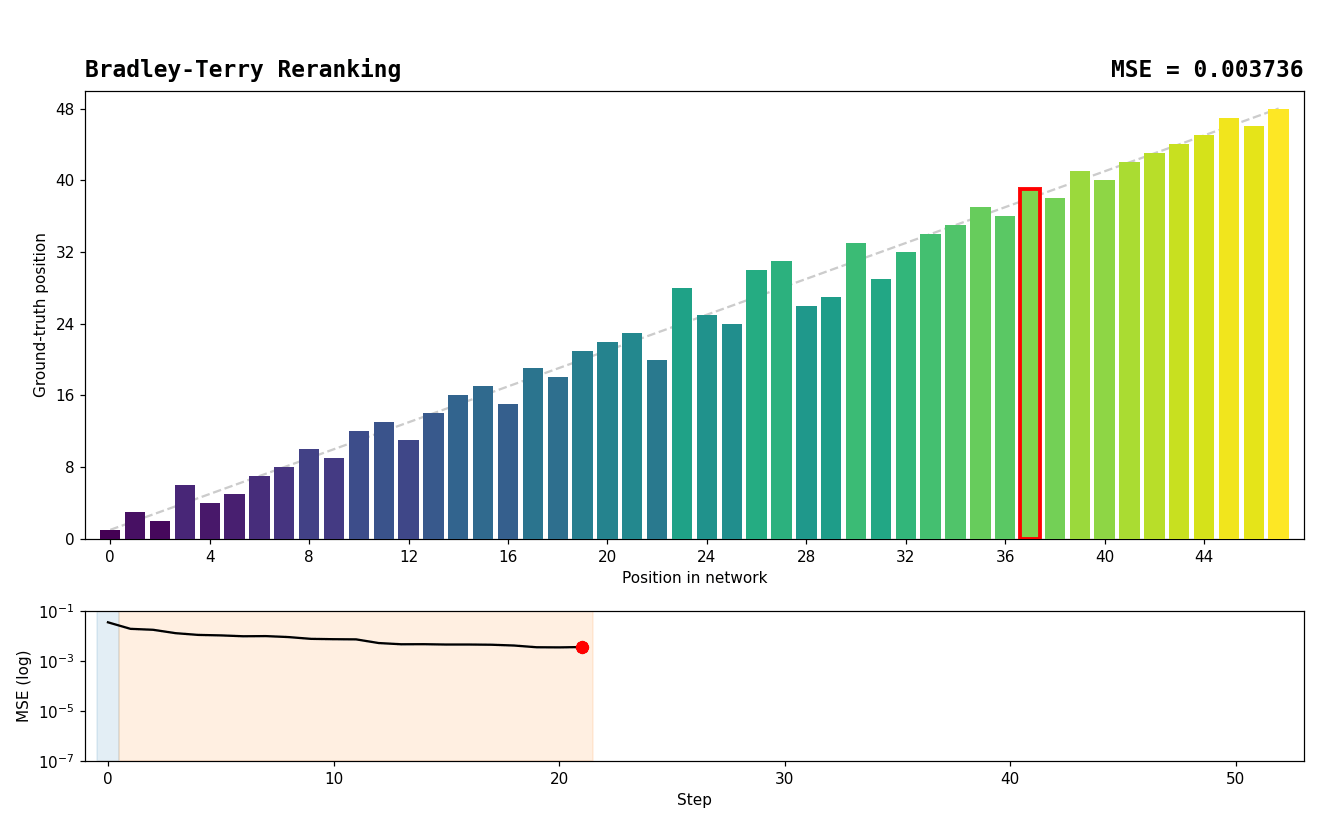}
        \caption{After BT reranking.}
    \end{subfigure}
    \vspace{-2pt}
    \begin{subfigure}[t]{\textwidth}
        \centering
        \includegraphics[width=0.72\textwidth]{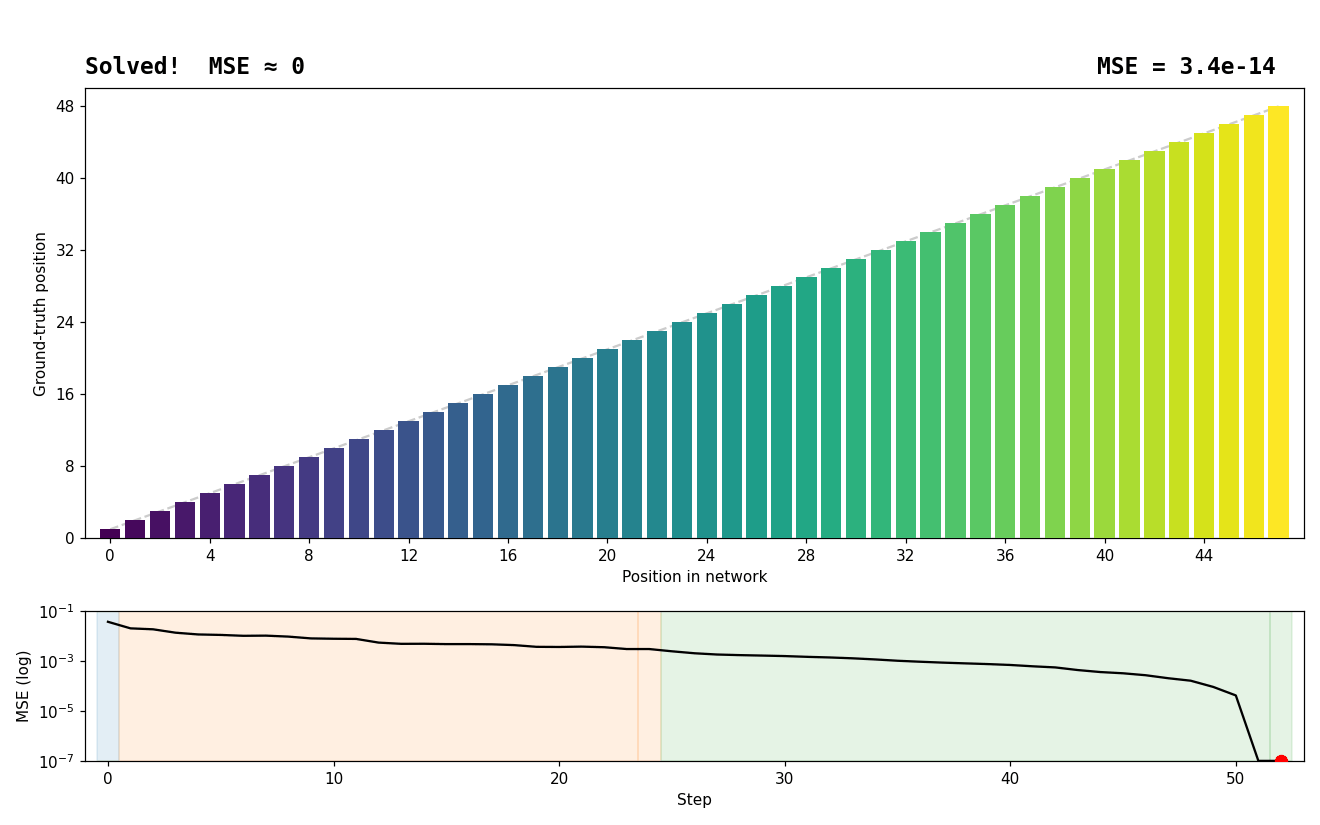}
        \caption{Solved. }
    \end{subfigure}
    \caption{\textbf{Ordering pipeline progression.}
    Bar height = ground-truth position (1--48)}
    \label{fig:sorting_progression}
\end{figure}

\newpage 
\section{Analysis}
\label{sec:analysis}
\subsection{Bradley--Terry Is Not a Necessity}
\label{sec:ablation_hc}

In practice, we find that transitivity violations are negligible and, due to strong initializations, hill-climbing alone could recover the solution. \textbf{Direct MSE hill-climbing from the delta-norm seed converges to MSE $= 0$ on its own}, without pairwise comparisons or Bradley--Terry (BT) ranking. Starting from MSE $= 0.036794$, bubble repair converges in 13 rounds with 122 total swaps (Table~\ref{tab:hc_convergence}).

\begin{table}[H]
    \centering
    \caption{\textbf{Hill-climb from delta-norm seed (no BT).}
    Converges to exact solution in 13 rounds / 122 swaps.}
    \label{tab:hc_convergence}
    \begin{tabular}{cccc}
        \toprule
        \textbf{Round} & \textbf{Swaps} & \textbf{MSE} & \textbf{Cum.\ swaps} \\
        \midrule
        0 (seed) & -- & 0.036794 & 0 \\
        1 & 23 & 0.021690 & 23 \\
        2 & 20 & 0.012919 & 43 \\
        3 & 21 & 0.008559 & 64 \\
        4 & 15 & 0.005873 & 79 \\
        5 & 11 & 0.004097 & 90 \\
        6 & 9  & 0.002745 & 99 \\
        7 & 12 & 0.001012 & 111 \\
        8 & 5  & 0.000380 & 116 \\
        9 & 2  & 0.000247 & 118 \\
        10 & 1 & 0.000182 & 119 \\
        11 & 1 & 0.000118 & 120 \\
        12 & 1 & 0.000052 & 121 \\
        13 & 1 & \textbf{0.000000} & 122 \\
        \bottomrule
    \end{tabular}
\end{table}

\subsection{Hill-climb from $\|W_{\text{out}}\|_F$ seed has better convergence}
\label{sec:ablation_wout}

Starting from a \emph{data-free} seed ordering based on $\|W_{\text{out}}\|_F$,
bubble repair reaches the solution in 6 rounds with 72 swaps
(Table~\ref{tab:hc_wout_convergence}). Despite the higher initial MSE
($0.075851$ vs.\ $0.036794$), convergence is $2\times$ faster in rounds and
requires $1.7\times$ fewer swaps.

\begin{table}[H]
    \centering
    \caption{\textbf{Hill-climb from $\|W_{\text{out}}\|_F$ seed (data-free, no BT).}
    Starting from MSE $= 0.075851$, converges in 6 rounds / 72 swaps.
    Despite higher initial MSE, convergence is $2\times$ faster in rounds
    and $1.7\times$ fewer swaps, and is \emph{data-free}.}
    \label{tab:hc_wout_convergence}
    \begin{tabular}{cccc}
        \toprule
        \textbf{Round} & \textbf{Swaps} & \textbf{MSE} & \textbf{Cum.\ swaps} \\
        \midrule
        0 (seed) & -- & 0.075851 & 0 \\
        1 & 30 & 0.005936 & 30 \\
        2 & 16 & 0.003397 & 46 \\
        3 & 12 & 0.001405 & 58 \\
        4 & 9  & 0.000568 & 67 \\
        5 & 4  & 0.000223 & 71 \\
        6 & 1  & \textbf{0.000000} & 72 \\
        \bottomrule
    \end{tabular}
\end{table}

The BT step mainly improves the \emph{starting point}: from a BT ordering,
bubble repair converges in 5 rounds (Table~\ref{tab:repair}) rather than 13.
However, this comes at the cost of 1{,}128 pairwise swap evaluations (forward
passes). Overall, direct hill-climbing and BT+repair have comparable end-to-end
cost in this setting.

\section{Discussion}

\subsection{What is this data?}
I have spent some time trying to figure out what this data could be and how these features were processed. This turned out to be very difficult, if not impossible. I am very curious to find out (see Figure~\ref{fig:sorted_48_features}). It might just be synthetic.

\subsection{Unreasonable effectiveness of delta norm}
Why does the delta norm work so well as a seed initialization? Are there  better seed orderings? what are their theoretical underpinnings?

\subsection{Convergence}
Why does $\|W_{\text{out}}\|_F$ seed has better convergence despite being a worse initialization compared to the delta norm? Is the geometry of the loss landscape nice? Are there local minima where hill-climbing would get stuck?

\bibliographystyle{plainnat}
\bibliography{references}

@inproceedings{he2016deep,
  title={Deep residual learning for image recognition},
  author={He, Kaiming and Zhang, Xiangyu and Ren, Shaoqing and Sun, Jian},
  booktitle={Proceedings of the IEEE Conference on Computer Vision and Pattern Recognition},
  pages={770--778},
  year={2016}
}

@article{bradley1952rank,
  title={Rank analysis of incomplete block designs: {I}. The method of paired comparisons},
  author={Bradley, Ralph Allan and Terry, Milton E},
  journal={Biometrika},
  volume={39},
  number={3/4},
  pages={324--345},
  year={1952},
  publisher={JSTOR}
}

@article{hunter2004mm,
  title={{MM} algorithms for generalized {B}radley--{T}erry models},
  author={Hunter, David R},
  journal={The Annals of Statistics},
  volume={32},
  number={1},
  pages={384--406},
  year={2004},
  publisher={Institute of Mathematical Statistics}
}

@inproceedings{feng2022rank,
  title={Rank diminishing in deep neural networks},
  author={Feng, Ruili and Zheng, Kecheng and Huang, Yukun and Zhao, Deli and Jordan, Michael and Zha, Zheng-Jun},
  booktitle={Advances in Neural Information Processing Systems},
  volume={35},
  pages={33054--33065},
  year={2022}
}

@article{kuhn1955hungarian,
  title={The {H}ungarian method for the assignment problem},
  author={Kuhn, Harold W},
  journal={Naval Research Logistics Quarterly},
  volume={2},
  number={1-2},
  pages={83--97},
  year={1955},
  publisher={Wiley Online Library}
}

@inproceedings{veit2016residual,
  title={Residual networks behave like ensembles of relatively shallow networks},
  author={Veit, Andreas and Wilber, Michael J and Belongie, Serge},
  booktitle={Advances in Neural Information Processing Systems},
  volume={29},
  year={2016}
}

@article{saxe2014exact,
  title={Exact solutions to the nonlinear dynamics of learning in deep linear networks},
  author={Saxe, Andrew M and McClelland, James L and Ganguli, Surya},
  journal={International Conference on Learning Representations},
  year={2014}
}

@inproceedings{pennington2017resurrecting,
  title={Resurrecting the sigmoid in deep learning through dynamical isometry: theory and practice},
  author={Pennington, Jeffrey and Schoenholz, Samuel S and Ganguli, Surya},
  booktitle={Advances in Neural Information Processing Systems},
  volume={30},
  year={2017}
}

@inproceedings{xiao2018dynamical,
  title={Dynamical isometry and a mean field theory of {CNN}s: How to train 10,000-layer vanilla convolutional neural networks},
  author={Xiao, Lechao and Bahri, Yasaman and Sohl-Dickstein, Jascha and Schoenholz, Samuel S and Pennington, Jeffrey},
  booktitle={International Conference on Machine Learning},
  pages={5393--5402},
  year={2018},
  organization={PMLR}
}

@article{kingma2015adam,
  title={Adam: A method for stochastic optimization},
  author={Kingma, Diederik P and Ba, Jimmy},
  journal={International Conference on Learning Representations},
  year={2015}
}

\newpage

\newpage
\appendix
\section*{Appendix}

\section{Engineer's Induction}
\label{sec:toy}

We train our own 48-block network with identical
architecture on the same data (Adam~\cite{kingma2015adam}, \texttt{lr = 1e-4}, 200 epochs,
MSE~$\approx 0.0006$). We observe a strong negative diagonal structure.
(Figure~\ref{fig:toy_matrices}), matching the original model
(Figure~\ref{fig:all48}).

\begin{figure}[H]
    \centering
    \includegraphics[width=\textwidth]{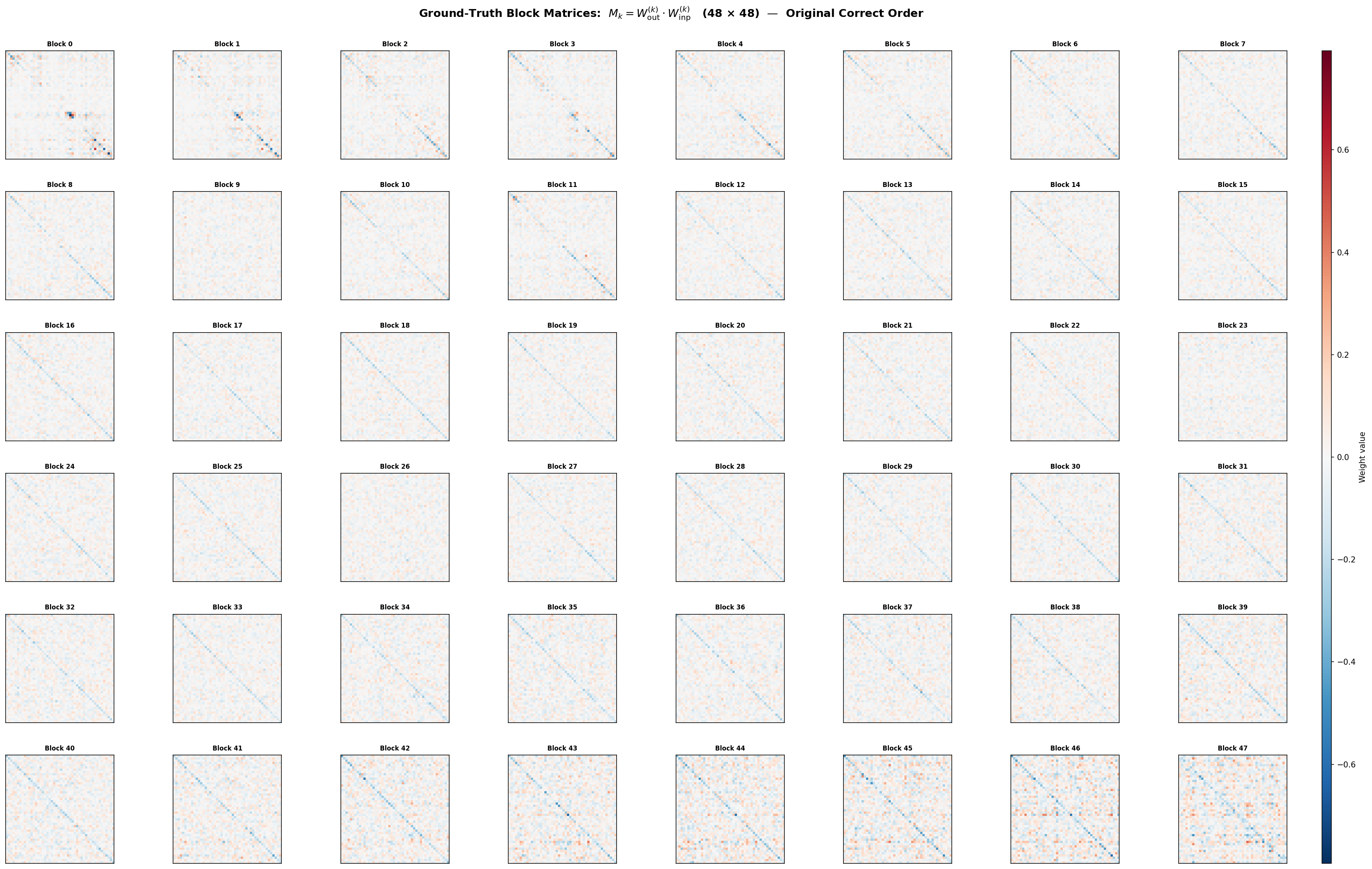}
    \caption{\textbf{Product matrices from a freshly trained model.}
    All 48 matrices $W_{\mathrm{out}}^{(k)} W_{\mathrm{in}}^{(k)}$,
    sorted by Frobenius norm. }
    \label{fig:toy_matrices}
\end{figure}

\section{Why Negative Diagonal Structure?}
\label{app:jacobian}

We derive the negative diagonal structure of
$W_{\mathrm{out}}^{(k)} W_{\mathrm{in}}^{(k)}$ from dynamic
isometry~\cite{saxe2014exact,pennington2017resurrecting,xiao2018dynamical},
then show why it provides a signal for correct pairings.

\subsection{Setup}

Recall each residual block (Eq.~\ref{eq:block}):
$f(x) = x + r(x)$ with
$r(x) = W_{\mathrm{out}}\,\ReLU(W_{\mathrm{in}}\,x + b_{\mathrm{in}}) + b_{\mathrm{out}}$,
where $W_{\mathrm{in}} \in \mathbb{R}^{h \times d}$,
$W_{\mathrm{out}} \in \mathbb{R}^{d \times h}$
($d{=}48$, $h{=}96$).

Let $D(x) := \diag(\mathbf{1}[W_{\mathrm{in}}\,x + b_{\mathrm{in}} > 0]) \in \{0,1\}^{h \times h}$
be the ReLU gating matrix.
The Jacobians are
\begin{equation}
    J_r(x) = W_{\mathrm{out}}\,D(x)\,W_{\mathrm{in}}, \qquad
    J_f(x) = I_d + J_r(x).
\end{equation}

\subsection{Main Result}

\begin{proposition}[Negative diagonal structure]
\label{prop:jacobian}
If the block satisfies dynamic
isometry~\cite{pennington2017resurrecting} in expectation,
$\mathbb{E}_x[J_f(x)^\top J_f(x)] = I_d$,
with non-trivial residual, and the ReLU activation probability is
$\tfrac{1}{2}$ per unit (i.e.\ $\mathbb{E}_x[D(x)] = \tfrac{1}{2}I_h$),
then $\tr(W_{\mathrm{out}}\,W_{\mathrm{in}}) < 0$.
\end{proposition}

\begin{proof}
\textbf{Step 1 (Isometry constraint).}
Expand $\mathbb{E}_x[J_f^\top J_f] = I_d$.
Since $J_f = I + J_r$:
\begin{equation}
    \mathbb{E}_x\!\big[(I + J_r)^\top(I + J_r)\big] = I
    \quad\Longrightarrow\quad
    \mathbb{E}_x\!\big[J_r + J_r^\top + J_r^\top J_r\big] = 0.
\end{equation}

\textbf{Step 2 (Trace).}
Take the trace. Using $\tr(J_r^\top) = \tr(J_r)$ and
$\tr(J_r^\top J_r) = \|J_r\|_F^2$:
\begin{equation}
    2\,\mathbb{E}_x\!\big[\tr(J_r)\big] + \mathbb{E}_x\!\big[\|J_r\|_F^2\big] = 0
    \quad\Longrightarrow\quad
    \mathbb{E}_x\!\big[\tr(J_r)\big] = -\tfrac{1}{2}\,\mathbb{E}_x\!\big[\|J_r\|_F^2\big].
    \label{eq:trace_neg}
\end{equation}

\textbf{Step 3 (Evaluate the left-hand side).}
Since $W_{\mathrm{out}}$ and $W_{\mathrm{in}}$ are fixed and only $D(x)$
depends on $x$:
\begin{equation}
    \mathbb{E}_x\!\big[\tr(J_r(x))\big]
    = \tr\!\big(W_{\mathrm{out}}\,\mathbb{E}_x[D(x)]\,W_{\mathrm{in}}\big).
\end{equation}
Each diagonal entry of $D(x)$ is either zero or one. By the activation-probability assumption ($\mathbb{E}[D(x)] = \tfrac{1}{2}I_h$; verified empirically in Figure~\ref{fig:relu_fraction} and Figure~\ref{fig:spectrum}):
\begin{equation}
    \mathbb{E}_x\!\big[\tr(J_r(x))\big]
    = \tfrac{1}{2}\,\tr(W_{\mathrm{out}}\,W_{\mathrm{in}}).
\end{equation}

\textbf{Step 4 (Combine).}
Substituting into Eq.~\eqref{eq:trace_neg}:
\begin{equation}
    \tfrac{1}{2}\,\tr(W_{\mathrm{out}}\,W_{\mathrm{in}})
    = -\tfrac{1}{2}\,\mathbb{E}_x\!\big[\|J_r(x)\|_F^2\big]
    \quad\Longrightarrow\quad
    \boxed{\tr(W_{\mathrm{out}}\,W_{\mathrm{in}})
    = -\,\mathbb{E}_x\!\big[\|J_r(x)\|_F^2\big] < 0.}
\end{equation}
Strict inequality holds because $J_r(x) \neq 0$
(the residual is non-trivial), so
$\mathbb{E}_x[\|J_r(x)\|_F^2] > 0$.
\end{proof}

Empirically, $\tr(W_{\mathrm{out}}^{(k)} W_{\mathrm{in}}^{(k)}) \in [-13.5,\, -7.4]$
across all 48 blocks (Figure~\ref{fig:all48}).

\subsection{Why Pairing Works}
\label{app:pairing_theory}

\paragraph{Correct pair: $-\epsilon\,I$ structure.}
Proposition~\ref{prop:jacobian} guarantees $\tr(M_{ii}) < 0$, so
we can decompose a correctly paired product as
\begin{equation}
    M_{ii}
    := W_{\mathrm{out}}^{(i)}\,W_{\mathrm{in}}^{(i)}
    = -\epsilon\,I_d + E, \qquad
    \epsilon := \tfrac{1}{d}\,|\tr(M_{ii})| > 0,
    \label{eq:eps_decomp}
\end{equation}
where $E := M_{ii} + \epsilon\,I_d$ captures off-diagonal structure and
diagonal variation. By construction $\tr(E) = 0$, so
$\tr(M_{ii}) = -\epsilon\,d$.
For the Frobenius norm, expand using $(M_{ii})_{lm} = -\epsilon\,(I_d)_{lm} + E_{lm}$:
\begin{equation}
    \|M_{ii}\|_F^2
    = \sum_{l,m}\!\big({-}\epsilon\,(I_d)_{lm} + E_{lm}\big)^2
    = \epsilon^2\!\underbrace{\sum_{l,m}\!(I_d)_{lm}^2}_{= \,d}
    \;-\; 2\epsilon\!\underbrace{\sum_l E_{ll}}_{=\,\tr(E)\,=\,0}
    \;+\; \|E\|_F^2
    = \epsilon^2 d + \|E\|_F^2.
\end{equation}
The diagonal dominance ratio for a correct pair is therefore
\begin{equation}
    d(i,i)
    = \frac{|\tr(M_{ii})|}{\|M_{ii}\|_F}
    = \frac{\epsilon\,d}{\sqrt{\epsilon^2 d + \|E\|_F^2}}
    = \frac{\sqrt{d}}{\sqrt{1 + \|E\|_F^2/(\epsilon^2 d)}}\,
    \leq \sqrt{d}.
    \label{eq:d_correct}
\end{equation}
The upper bound $\sqrt{d} \approx 6.93$ is attained only for scaled identity. Empirically the correct-pair
ratios reach $d(i,i) \in [1.76,\, 3.28]$, reflecting
non-negligible off-diagonal energy $\|E\|_F$.

\paragraph{Incorrect pair: random matrix baseline.}
For an incorrect pairing $(i, j)$ with $i \neq j$, the matrices
$W_{\mathrm{out}}^{(j)}$ and $W_{\mathrm{in}}^{(i)}$ come from independently
trained blocks. Model their entries as independent zero-mean random variables
with variances $\sigma_{\mathrm{out}}^2$ and $\sigma_{\mathrm{in}}^2$.
Each entry of $M_{ij} = W_{\mathrm{out}}^{(j)} W_{\mathrm{in}}^{(i)}$ is a
sum of $h$ independent zero-mean products:
\begin{equation}
    (M_{ij})_{lm} = \sum_{k=1}^{h} (W_{\mathrm{out}}^{(j)})_{lk}\,(W_{\mathrm{in}}^{(i)})_{km},
    \qquad
    \mathrm{Var}\!\big((M_{ij})_{lm}\big)
    = h\,\sigma_{\mathrm{out}}^2\,\sigma_{\mathrm{in}}^2.
\end{equation}
The trace sums $d$ diagonal entries. Entry $(M_{ij})_{ll}$ depends on
row $l$ of $W_{\mathrm{out}}^{(j)}$ and column $l$ of $W_{\mathrm{in}}^{(i)}$;
different values of $l$ use disjoint rows and columns, so the diagonal
entries are mutually independent. Therefore
\begin{equation}
    \mathbb{E}\!\big[\tr(M_{ij})\big] = 0,
    \qquad
    \mathrm{sd}\!\big(\tr(M_{ij})\big)
    = \sigma_{\mathrm{out}}\,\sigma_{\mathrm{in}}\,\sqrt{dh}\,.
\end{equation}
The expected squared Frobenius norm sums all $d^2$ entries:
$\mathbb{E}[\|M_{ij}\|_F^2] = d^2 h\,\sigma_{\mathrm{out}}^2\,\sigma_{\mathrm{in}}^2$.
The typical diagonal dominance ratio scales as
\begin{equation}
    \frac{\mathrm{sd}\!\big(\tr(M_{ij})\big)}
         {\sqrt{\mathbb{E}[\|M_{ij}\|_F^2]}}
    = \frac{\sigma_{\mathrm{out}}\,\sigma_{\mathrm{in}}\,\sqrt{dh}}
           {\sigma_{\mathrm{out}}\,\sigma_{\mathrm{in}}\,d\sqrt{h}}
    = \frac{1}{\sqrt{d}}
    \approx 0.14.
    \label{eq:d_random}
\end{equation}

\paragraph{Separation.}
The correct-pair score (Eq.~\ref{eq:d_correct}) is around $\sqrt{d}\approx6.93$ , while the incorrect-pair score
(Eq.~\ref{eq:d_random}) is ${\sim}1/\sqrt{d} \approx 0.14$.
Empirically: $d(i,i) \geq 1.76$ vs.\ $d(i,j) \leq 0.58$
(Figure~\ref{fig:pairing_histogram}).


\section{Additional Figures}
\label{app:figures}

\begin{figure}[ht]
    \centering
    \includegraphics[width=0.7\textwidth]{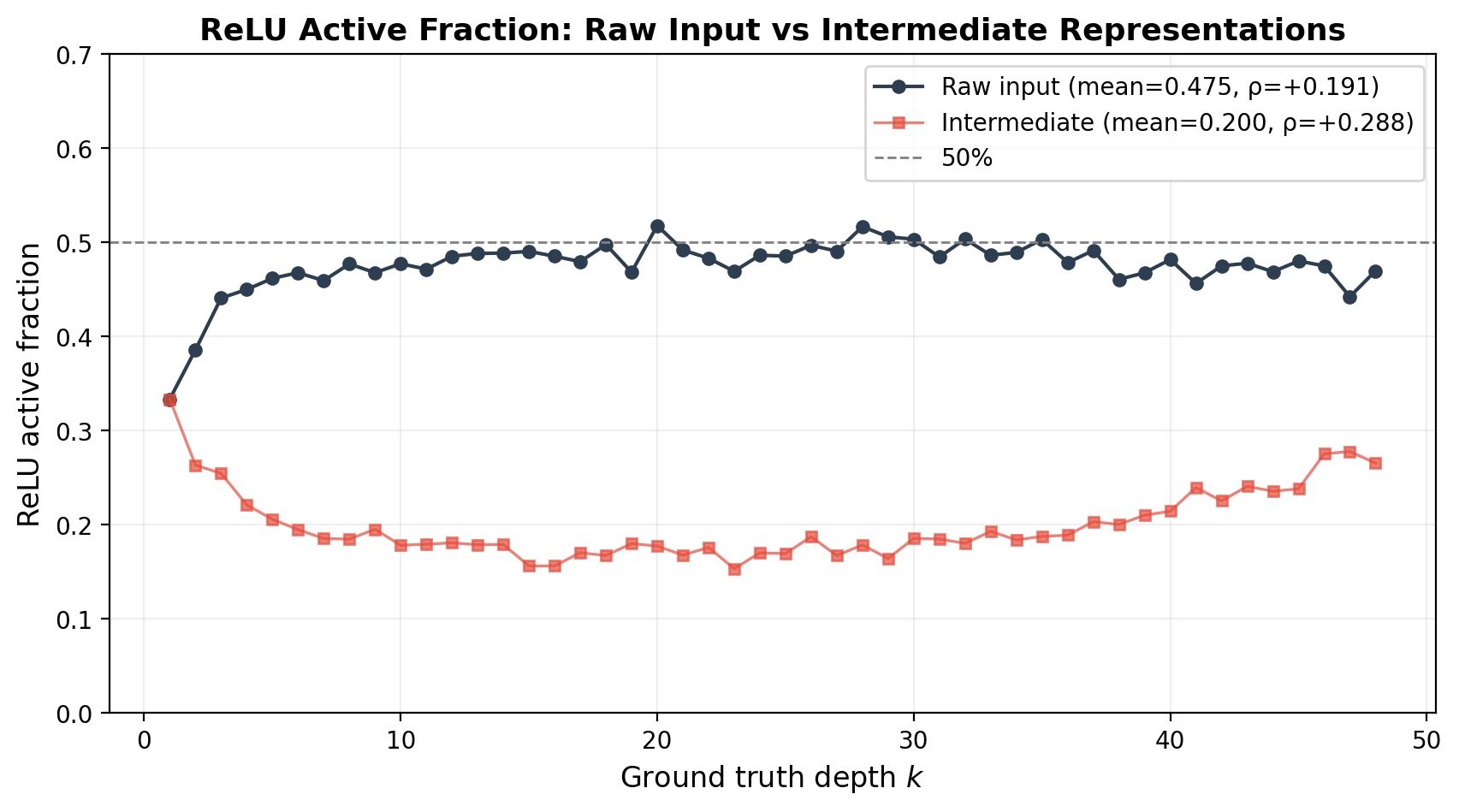}
    \caption{\textbf{ReLU activation fraction across blocks.}
    Per-block ReLU activation probability}
    \label{fig:relu_fraction}
\end{figure}

\begin{figure}[ht]
    \centering
    \includegraphics[width=\textwidth]{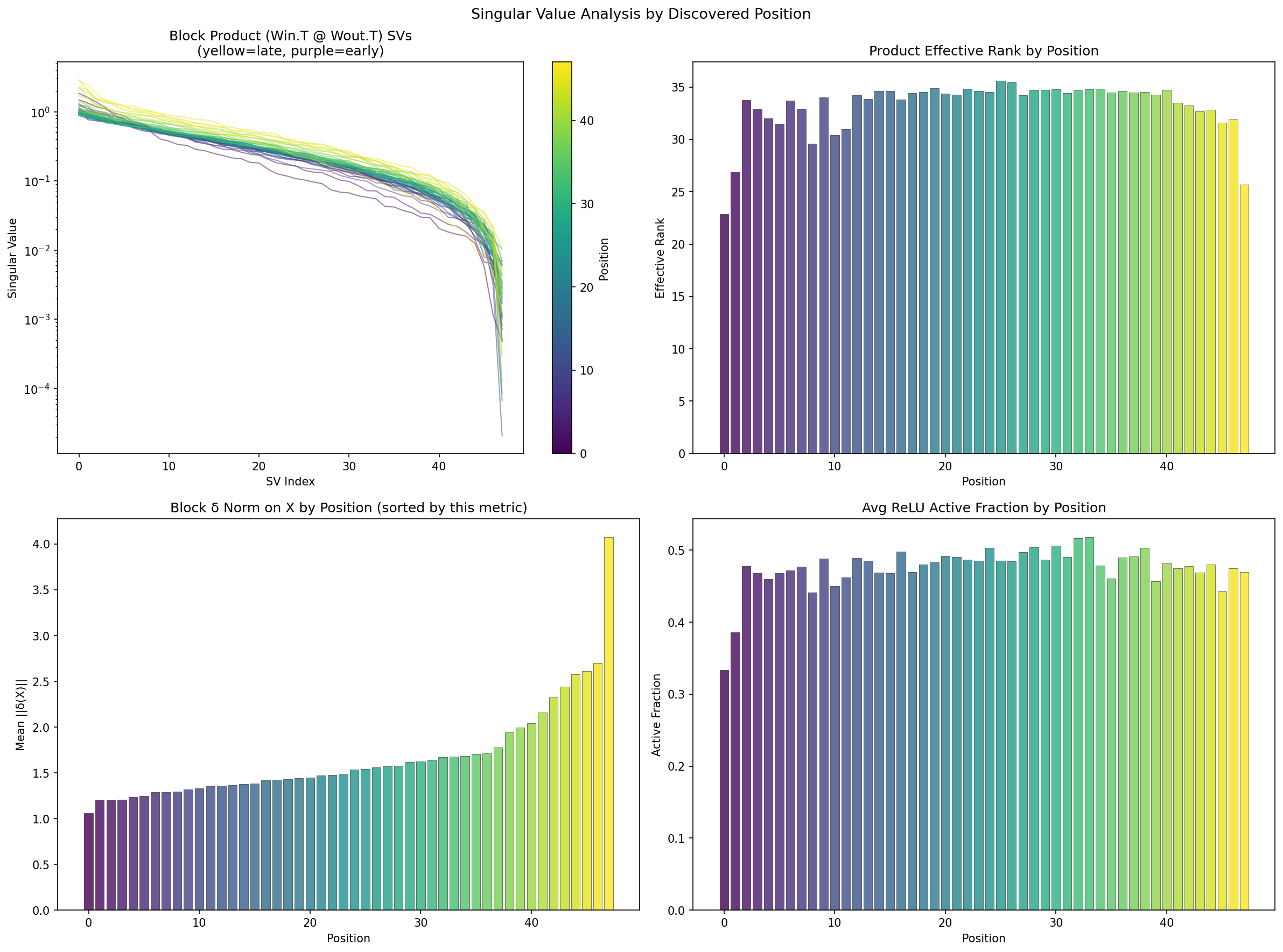}
    \caption{\textbf{Block properties by recovered position.}
    Top-left: singular value spectra of $W_{\mathrm{out}} W_{\mathrm{in}}$,
    colored by position (purple~= early, yellow~= late).
    Top-right: effective rank vs.\ position.
    Bottom-left: delta-norm increases with depth.
    Bottom-right: ReLU active fraction}
    \label{fig:spectrum}
\end{figure}

\end{document}